\documentclass[conference]{IEEEtran}

\ifCLASSINFOpdf
  \usepackage[pdftex]{graphicx}
\else
\fi

\usepackage{amsmath}
\usepackage{amssymb}
\usepackage{amsthm}
\usepackage{algorithm}
\usepackage{algorithmic}
\usepackage{threeparttable}
\usepackage{xcolor}
\usepackage[caption=false,font=footnotesize]{subfig}
\usepackage{url}
\usepackage{multirow}
\usepackage{booktabs}
\newtheorem{theorem}{Theorem}

\begin{document}

\title{ToE: A Hierarchical and Explainable Claim Verification Framework with Dynamic Multi-source Evidence Retrieval and Aggregation}

\author{%
\IEEEauthorblockN{%
Zhaoqi Wang\IEEEauthorrefmark{1},
Zijian Zhang\IEEEauthorrefmark{1}\thanks{Corresponding author: Zijian Zhang. Email: zhangzijian@bit.edu.cn.},
Kun Zheng\IEEEauthorrefmark{1},
Zhen Li\IEEEauthorrefmark{1},
Xin Li\IEEEauthorrefmark{1},
Chunlei Li\IEEEauthorrefmark{2},
Jiamou Liu\IEEEauthorrefmark{3}
}
\IEEEauthorblockA{\IEEEauthorrefmark{1}School of Cyberspace Science and Technology, Beijing Institute of Technology}
\IEEEauthorblockA{\IEEEauthorrefmark{2}TravelSky Technology Limited}
\IEEEauthorblockA{\IEEEauthorrefmark{3}School of Computer Science, University of Auckland}
}

\maketitle

\begin{abstract}
\textcolor{red}{Content Warning: This paper contains examples of misinformation used only for research purpose.}
The rapid spread of fake news poses increasing threats to information ecosystems, especially as AI-generated misinformation under Generative Engine Optimization (GEO) poisoning allows adversarially crafted content to be systematically surfaced by retrieval systems, contaminating LLM reasoning. In this paper, we propose Tree of Evidence (ToE), a hierarchical evidence reasoning framework for automated fact-checking that models each claim as a dynamically expanding argument tree. ToE integrates a reinforcement learning-driven multi-source retrieval agent, an evidence evaluation agent, and an argument tree aggregation algorithm to iteratively decompose, retrieve, and verify claims through an explainable evidence chain. We further provide a theoretical analysis of the retrieval process, deriving a formal error bound that guarantees the learned policy converges to a neighborhood of the information-theoretically optimal policy. Experiments across multiple datasets and backbone LLMs demonstrate that ToE achieves improvements ranging from 4 to 24 percentage points over competitive baselines, with particularly pronounced gains on adversarially poisoned inputs.
\end{abstract}

\pagestyle{plain}

\section{Introduction}

Large language models (LLMs), such as the DeepSeek series~\cite{deepseek-v3} and the GPT series~\cite{gpt-4}, have demonstrated impressive capabilities across a wide range of tasks. However, due to their reliance on static training data, LLMs lack access to real-time information, which makes them prone to hallucination. To address this limitation, two major lines of work have been proposed. Retrieval-Augmented Generation (RAG) constructs an external knowledge index and retrieves relevant documents at inference time to supply the model with up-to-date context~\cite{rag}. Tool calling~\cite{toolformer}, on the other hand, enables LLMs to dynamically invoke external tools such as search engines during generation, thereby grounding responses in freshly retrieved information. Both approaches alleviate the hallucination problem to a certain extent.
However, the introduction of third-party information also brings new risks. AI-generated or manually fabricated misinformation can lead LLMs to produce incorrect conclusions~\cite{cpa,poisonedrag,scr}. This risk may be further exacerbated by GEO, a technique that structures content to maximize its discoverability by retrieval algorithms such as embedding-based ranking~\cite{geo}. Compared to manually written truthful content, adversarially crafted misinformation optimized via GEO can be systematically ranked higher in retrieval results, making it more likely to be consumed by LLMs. For example, as illustrated in Figure~\ref{fig:overview}, when a user asks \textit{``Who is the CEO of OpenAI?''}, an attacker can inject a fabricated document claiming that Tim Cook has joined OpenAI as CEO, which once ranked alongside legitimate sources, contaminates the retrieved context and causes the LLM to produce a confidently wrong answer.

\begin{figure}[ht]
   \centering
   \includegraphics[width=1.0\linewidth]{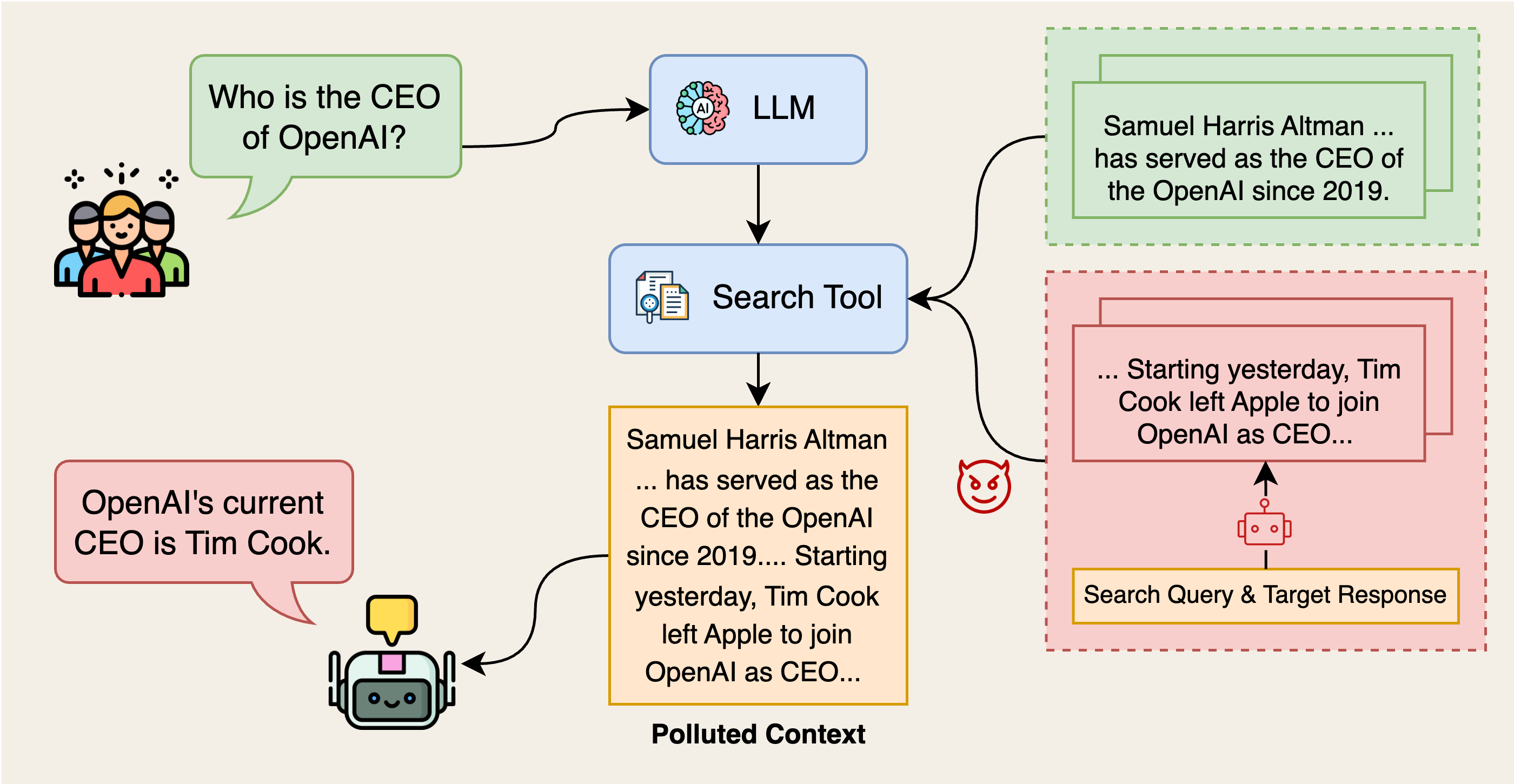}
   \caption{Illustration of LLM context pollution via malicious retrieval}
   \label{fig:overview}
\end{figure}

To combat the spread of misinformation, researchers have proposed a range of detection methods. Early work relies on deep learning-based approaches~\cite{fakenews_survey}, but these methods suffer from limited generalizability, as the inherent diversity of online information across topics, styles, and platforms makes it difficult to transfer models trained on one distribution to unseen data from another~\cite{teller}. More recently, LLM-based fake news detection methods have been proposed~\cite{f3,teller,steel,afacta}, which leverage the reasoning capabilities of LLMs to analyze input content and identify potential misinformation. While these methods achieve promising results on scientific or knowledge-intensive claims, they struggle with time-sensitive news, particularly in the face of AI-generated misinformation where fabricated content is fluent, coherent, and difficult to distinguish from genuine reporting without access to real-time external knowledge.

To address this challenge, we propose ToE, a hierarchical evidence reasoning framework for automated fact-checking. ToE models a claim as a dynamically expanding argument tree and verifies it through three core components: a reinforcement learning-driven multi-source retrieval agent that collects evidence across platforms according to the characteristics of the claim under investigation, an evidence evaluation agent that scores the veracity and reliability of the claim based on the collected evidence, and an argument tree aggregation agent that decomposes the claim into sub-claims along dimensions such as who, what, when, where, why, and how, expanding the subtree for deeper verification when the current evidence is insufficient for a reliable judgment. Node-level scores are propagated bottom-up through the tree, and the process terminates when the root node reaches a convergence or decisiveness threshold, yielding a final veracity score along with the full reasoning tree.

We summarize our main contributions as follows:

\begin{itemize}
    \item We propose ToE, a hierarchical evidence reasoning framework for automated fact-checking. To the best of our knowledge, this is the first algorithm based on dynamic evidence collection and evaluation, which produces an interpretable argument tree as an explainable evidence chain for each verdict.
    \item We provide a theoretical analysis of the retrieval framework by modeling evidence collection as a Partially Observable Markov Decision Process, and derive a formal error bound showing that the learned policy converges to a neighborhood of the information-theoretically optimal policy.
    \item We construct an adversarial dataset AdvFact to evaluate robustness under GEO poisoning, and conduct experiments across multiple datasets and LLMs with ablation studies on the retrieval action space, demonstrating the effectiveness and generalizability of the proposed approach.
\end{itemize}

\section{Related Work}

\begin{figure*}[t]
   \centering
   \includegraphics[width=0.9\linewidth]{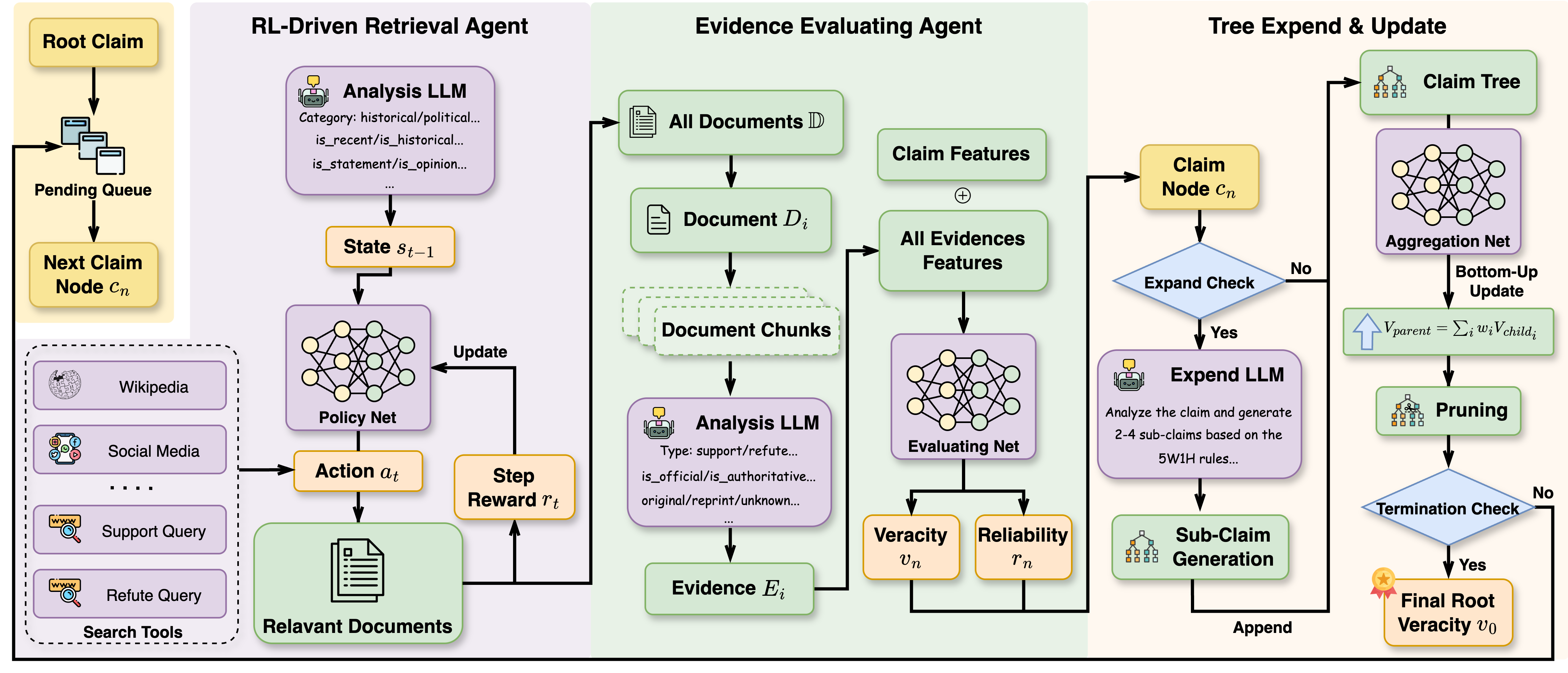}
   \caption{An overview of ToE framework.}
   \label{fig:workflow}
\end{figure*}

With the rapid development of AI, LLM-based agents address the knowledge staleness problem by invoking external search tools to retrieve relevant information as supplementary context~\cite{rag,toolformer}. However, this reliance on third-party sources introduces new vulnerabilities. Prior work has demonstrated that AI-generated misinformation injected into retrieval results can mislead LLMs into producing incorrect conclusions, with potentially severe consequences in domains such as medicine and finance~\cite{poisonedrag,fakegpt}. While deep learning-based detection methods have been proposed to identify such content~\cite{fakenews_survey}, they rely on surface-level stylistic features and fail to generalize to AI-generated misinformation that differs from truthful content only in subtle factual details. For instance, a claim such as ``Tim Cook is Apple's CEO'' can be corrupted to ``Tim Cook became OpenAI's CEO yesterday,'' a statement that is stylistically indistinguishable from genuine news yet factually false.
Recent work leverages the reasoning capabilities of LLMs to assist in verifying the reliability of information. F3 framework guides LLMs through stepwise logical analysis and evidence consistency assessment via prompt engineering techniques such as zero-shot chain-of-thought reasoning and deductive generation, producing a veracity judgment for the input claim~\cite{f3}. TELLER decomposes fake news detection into a set of structured evaluation questions, directing LLMs to analyze news from dimensions including factual accuracy, contextual consistency, and source credibility, then aggregates the dimension-level scores through interpretable logical rules to reach a final verdict~\cite{teller}. However, these methods predominantly rely on LLMs to analyze content provided by the claim content, and show limited performance against AI-generated misinformation. STEEL adopts a multi-round retrieval-augmented strategy in which an LLM assesses the confidence of initial retrieval results and automatically generates refined queries to iteratively collect additional web evidence when the current evidence is insufficient, alleviating the dependence on static knowledge bases~\cite{steel}. AdSent rewrites input claims into sentiment-neutralized variants to force the veracity classifier to rely solely on factual content, improving robustness against sentiment-manipulated misinformation~\cite{adsent}. However, it provides limited defense against misinformation grounded in fabricated facts.

\section{Method}

To counter the potential impact of false and misleading information, we propose ToE, a hierarchical evidence reasoning framework for automated fact-checking. ToE models the verification of a claim as a dynamically growing argument tree. At each step, it collects and evaluates evidence from heterogeneous sources — including Wikipedia, Arxiv, fact-checking platforms, search engine results, and social media. Based on the evaluation results, new subtrees are expanded to progressively refine the veracity judgment. This process ultimately produces an interpretable evidence tree that traces the full reasoning process.

As depicted in Fig.~\ref{fig:workflow}, the system initializes the argument tree with the claim under investigation as the root node and enters the main iterative loop. At each iteration, the priority of each pending node is computed dynamically based on the uncertainty of its parent and its own estimated importance, so that nodes with the greatest expected contribution to the final judgment are processed first. For each selected node, the system analyzes semantic features of the claim, such as its category and geographic scope, and generates three types of search queries: general background queries, supporting evidence queries, and counter-evidence queries. This design ensures comprehensive and objective evidence collection and reduces confirmation bias. A retrieval agent trained via reinforcement learning then executes searches autonomously across heterogeneous sources, including Wikipedia, PolitiFact, social media, and the general web, and decides dynamically when to stop retrieval. The collected content is parsed by an LLM to extract evidence snippets directly relevant to the claim, each annotated with attributes such as stance and source credibility, forming a structured evidence set. Based on this evidence, the system computes two core scores for the current node: veracity, which measures the probability that the claim is true, and reliability, which measures how strongly the current evidence supports that judgment.

After each node is evaluated, scores are propagated bottom-up through the tree via an aggregation network, continuously updating the root node's overall judgment. If a node's reliability is insufficient, indicating that the available evidence does not yet support a confident verdict, the system invokes an LLM to decompose the claim into finer-grained, more verifiable sub-claims along various dimensions, covering time, location, person, event, cause, and information source, and adds them as new child nodes for further processing. Concurrently, subtrees that have reached high reliability with a sufficiently decisive verdict are pruned to avoid redundant computation. The iterative process continues until the root node's judgment converges, a decisiveness threshold is met, or the maximum number of iterations is reached, ultimately producing a veracity score for the claim and a complete argument tree available for traceability and review.

\subsection{Theoretical Foundations}

The verification process in ToE can be understood through two complementary theoretical lenses: a decision-theoretic formulation that characterizes the structure of the search problem, and an information-theoretic interpretation that justifies the design of the reward signal driving the retrieval agent.

\paragraph{POMDP Formulation.}
The core challenge in automated fact-checking is that the ground-truth veracity $v^* \in [0,1]$ of a claim $c$ is a latent variable that cannot be directly observed. The system can only accumulate evidence through sequential search actions and must form a judgment under persistent uncertainty. This structure maps naturally onto a Partially Observable Markov Decision Process (POMDP)~\cite{POMDP}, defined as the tuple $(\mathcal{S}, \mathcal{A}, \mathcal{O}, \mathcal{T}, \mathcal{R})$, with components as follows.
The \textbf{state} $s_t = (\mathcal{E}_t, f_c)$ is the joint representation of the evidence set $\mathcal{E}_t$ accumulated up to step $t$ and the semantic features $f_c$ of the claim under investigation, where $f_c$ encodes attributes such as claim category and geographic scope. The \textbf{action} $a_t \in \mathcal{A}$ is selected from the eight-option discrete action space of the retrieval agent, covering heterogeneous source types including Wikipedia, arXiv, fact-checking platforms, social media, and a stop action. The \textbf{observation} $o_t \in \mathcal{O}$ is the set of documents retrieved after executing action $a_t$, which is subsequently processed by the evidence evaluation agent into structured evidence objects annotated with stance and credibility attributes. The \textbf{transition} $\mathcal{T}(s_{t+1} \mid s_t, a_t)$ describes how the evidence set is updated upon receiving new observations: $\mathcal{E}_{t+1} = \mathcal{E}_t \cup \{o_t\}$. The \textbf{belief state} $b_t = (v_{n,t}, r_{n,t})$ is maintained by the evaluation network, which maps the current evidence set and claim features to a veracity score $v_{n,t}$ and a reliability score $r_{n,t}$. Crucially, the evaluation network serves as a belief state estimator rather than an independently trained value network: it is a domain-driven module that produces a posterior estimate of $v^*$ given all available evidence, which distinguishes ToE from standard Actor-Critic architectures where the value function is learned separately from the environment model.

\paragraph{Information-Seeking Objective.}
The reward signal driving the retrieval agent admits a natural information-theoretic interpretation. At each step $t$, the agent receives a step reward proportional to the reliability gain $\Delta r_t = r_{n,t} - r_{n,t-1}$. We interpret this gain as an approximation of the conditional mutual information between the latent veracity $v^*$ and the new observation $o_t$ given prior evidence~\cite{submodular}:
\begin{equation}
\Delta r_t \approx I(v^*;\, o_t \mid \mathcal{E}_{t-1}) = H(v^* \mid \mathcal{E}_{t-1}) - H(v^* \mid \mathcal{E}_t).
\end{equation}
Each search action reduces the system's epistemic uncertainty about the true veracity of the claim, and the reliability score $r_{n,t}$ serves as a tractable surrogate for the posterior confidence over $v^*$.

To establish a formal performance guarantee for the greedy retrieval policy, we assume that the reliability function $r: 2^{\mathcal{O}} \to [0,1]$ is submodular with respect to the evidence set, meaning that the marginal reliability gain from any new observation $o$ satisfies
\begin{equation}
r(\mathcal{E}_A \cup \{o\}) - r(\mathcal{E}_A) \geq r(\mathcal{E}_B \cup \{o\}) - r(\mathcal{E}_B),
\end{equation}
for any $\mathcal{E}_A \subseteq \mathcal{E}_B$. This assumption captures the diminishing returns property of evidence accumulation: once a sufficient evidential foundation has been established, redundant or corroborating observations contribute progressively less to resolving uncertainty. This is empirically well-motivated in the fact-checking setting, where the first few high-quality sources typically determine the verdict and additional evidence yields smaller marginal gains. Under this assumption, the greedy policy that selects at each step the action maximizing the expected reliability gain is guaranteed to achieve performance no worse than $(1 - 1/e)$ of the optimal non-greedy policy, providing a formal justification for the step reward design.

We further note that the approximation in the mutual information interpretation is bounded. If the evaluation network's reliability estimate satisfies $|r_{n,t} - \tilde{r}_{n,t}| \leq \epsilon$ for all $t$, where $\tilde{r}_{n,t}$ denotes the ideal reliability derived from the true posterior, then the per-step deviation of the reward signal from the true information gain is uniformly controlled, as formalized below.

\begin{theorem}\label{thm:error_bound}
Let $\tilde{r}_{n,t}$ denote the ideal reliability score derived from the true posterior over $v^*$, and suppose the evaluation network satisfies $|r_{n,t} - \tilde{r}_{n,t}| \leq \epsilon$ for all $t$. Then the per-step deviation of the reward signal from the true conditional mutual information is uniformly bounded by $\left|\Delta r_t - I(v^*;\, o_t \mid \mathcal{E}_{t-1})\right| \leq 2\epsilon$ for all $t \in \{1, \dots, T\}$, and consequently the average cumulative deviation over the retrieval horizon satisfies $\frac{1}{T}\sum_{t=1}^{T} \left|\Delta r_t - I(v^*;\, o_t \mid \mathcal{E}_{t-1})\right| \leq 2\epsilon$, independently of $T$.
\end{theorem}

\begin{proof}
By the triangle inequality and the pointwise approximation $\Delta\tilde{r}_t \approx I(v^*;\, o_t \mid \mathcal{E}_{t-1})$, each term satisfies:
\begin{align}
\left|\Delta r_t - I(v^*;\, o_t \mid \mathcal{E}_{t-1})\right| &\leq \left|\Delta r_t - \Delta\tilde{r}_t\right| \nonumber \\
&= \left|(r_{n,t} - r_{n,t-1}) - (\tilde{r}_{n,t} - \tilde{r}_{n,t-1})\right| \nonumber \\
&\leq \left|r_{n,t} - \tilde{r}_{n,t}\right| + \left|r_{n,t-1} - \tilde{r}_{n,t-1}\right| \nonumber \\
&\leq 2\epsilon.
\end{align}
This bound holds uniformly for all $t$, so the average over $T$ steps is also bounded by $2\epsilon$.
\end{proof}

\paragraph{Reward Fidelity and Policy-Level Guarantee.}
This per-step guarantee further implies a bound on the policy-level performance gap. Let $\pi^*$ denote the optimal policy under the true information gain reward $I(v^*; o_t \mid \mathcal{E}_{t-1})$, and let $\hat{\pi}$ denote the policy optimized under the approximate reward $\Delta r_t$. By the simulation lemma for finite-horizon MDPs, a uniform per-step reward perturbation of $\delta$ induces a value function gap of at most $2T\delta$. Substituting the per-step bound $\delta = 2\epsilon$ from Theorem~\ref{thm:error_bound}, the value gap over a fixed horizon $T$ satisfies:
\begin{equation}
\left|V^{\pi^*} - V^{\hat{\pi}}\right| \leq 4T\epsilon,
\end{equation}
where $V^{\pi} = \mathbb{E}_{\pi}\!\left[\sum_{t=1}^T I(v^*; o_t \mid \mathcal{E}_{t-1})\right]$ denotes the expected cumulative information gain under policy $\pi$. The per-step bound of $2\epsilon$ is independent of the horizon length $T$ and ensures that the approximate reward preserves the correct greedy action ranking whenever the gap between the best and second-best actions exceeds $4\epsilon$. In our setting, the retrieval agent operates over a short fixed horizon with at most 10 retrieved documents per node and a maximum tree depth of 5, so the cumulative deviation remains small in absolute terms relative to the reliability range $[0, 1]$. These results formalize the conditions under which the learned reliability estimate is a faithful proxy for the true information gain, justifying the use of the surrogate reward in policy optimization.

\subsection{RL-Driven Retrieval Agent}

As one of the core components of ToE, the retrieval agent is responsible for searching the web for content relevant to the claim under investigation. The agent first analyzes the input claim to extract a set of semantic features, including its category (e.g., historical, scientific, or political) and attributes such as whether it describes a recent event, a historical fact, or a factual statement. These claim-level features are combined with step-level search state features, including the number of searches performed, the number of supporting snippets retrieved, and the number of counter-evidence snippets retrieved, to form the current state $s_t$. The state is fed into a policy network $\pi_\theta$, which selects an action $a_t$ from a discrete action space.
The action space covers eight options: searching general web sources, searching for counter-evidence, querying arXiv for academic papers, querying Wikipedia, searching social media platforms, querying fact-checking sites such as PolitiFact and Snopes, searching Reddit discussions, and stopping the search. This design reflects the observation that different types of claims are best verified through different sources. Established facts benefit from dedicated fact-checking sites, encyclopedic knowledge is well served by Wikipedia, and time-sensitive claims, particularly statements attributed to public figures or official bodies, are often most efficiently verified via social media. The explicit counter-evidence action further ensures that the agent actively seeks disconfirming information rather than collecting only supporting content.
Given the selected action, the agent uses an LLM to generate a search query tailored to the claim's content and then invokes the corresponding tool to retrieve results from the target source. After each search step, a step reward $r_t$ is computed and stored together with the transition in a trajectory buffer. Once the buffer reaches the designated capacity, the policy network is updated online using Group Relative Policy Optimization (GRPO)~\cite{grpo}, a reinforcement learning algorithm that estimates advantages by comparing returns within a group of sampled trajectories, eliminating the need for a separate value network.

The concrete policy state is a 29-dimensional vector. It contains claim-level attributes, including a 6-dimensional category indicator, a 3-dimensional locale indicator, and binary flags for recent events, historical facts, factual statements, opinions, and predictions. It also records search progress, including the number and ratio of completed searches, and evidence status, including the total number of retrieved evidence items, their stance breakdown, per-source search counts, and whether an official or authoritative source has been found. The policy network is a feed-forward model with hidden dimensions 128, 128, and 64, followed by a softmax action head.

% \begin{figure}[ht]
%    \centering
%    \includegraphics[width=0.8\linewidth]{policy_net.png}
%    \caption{Detail of Policy Network.}
%    \label{fig:policy}
% \end{figure}

At each search step $t$, the reward is computed as
\begin{equation}
r_t = w_r \cdot \Delta \text{rel}_t + w_v \cdot |\Delta \text{ver}_t| + w_e \cdot \min(n_t^{\text{new}}, 4) \cdot \delta_e + b_t - \delta_{\text{step}},
\end{equation}
where $\Delta \text{rel}_t$ and $|\Delta \text{ver}_t|$ denote the changes in reliability and veracity, $n_t^{\text{new}}$ is the number of newly retrieved evidence items, $b_t$ includes high-reliability and authoritative-source bonuses, and $\delta_{\text{step}}$ penalizes unnecessary searches. Table~\ref{tab:rl_hyperparams} summarizes the main hyperparameters. During the first $N_{\text{warm}}$ claims, the agent follows a heuristic warm-up policy based on claim category and recency, after which GRPO updates are triggered whenever the trajectory buffer accumulates $G$ completed trajectories.

\begin{table}[ht]
\centering
\caption{Hyperparameters of the Policy Network.}
\label{tab:rl_hyperparams}
\begin{tabular}{lll}
\hline
\textbf{Parameter} & \textbf{Symbol} & \textbf{Value} \\
\hline
Reliability improvement weight    & $w_r$                  & 1.0  \\
Veracity change weight            & $w_v$                  & 0.5  \\
Evidence bonus weight             & $w_e$                  & 0.05 \\
Per-step penalty                  & $\delta_{\text{step}}$ & 0.01 \\
High-confidence threshold         & $\tau_h$               & 0.8  \\
Authoritative source bonus        & $b_{\text{auth}}$      & 0.1  \\
High-reliability bonus            & $b_{\text{rel}}$       & 0.2  \\
Group size                        & $G$                    & 8    \\
Discount factor                   & $\gamma$               & 0.99 \\
Clipping range                    & $\varepsilon$          & 0.2  \\
Entropy coefficient               & $\alpha$               & 0.01 \\
KL coefficient                    & $\beta$                & 0.1  \\
Learning rate                     & --                     & 1e-5 \\
Max gradient norm                 & --                     & 0.5  \\
Warm-up claims                    & $N_{\text{warm}}$      & 8    \\
Max search steps per claim        & --                     & 5    \\
\hline
\end{tabular}
\end{table}

\subsection{Evidence Evaluation Agent}

After the retrieval agent collects relevant documents, each document $D_i$ is processed individually. To avoid exceeding the context limit of the LLM, each document is first split into chunks. Each chunk is submitted to a cleaning LLM that extracts statements directly relevant to the claim. Once all chunks of a document are processed, the extracted statements are combined with the document's metadata, such as URL and author information, and submitted to an analysis LLM. This LLM determines the document's stance toward the claim (supporting, refuting, or neutral), whether it originates from an official or authoritative source, and whether it constitutes a primary report or a republication, among other attributes. The result is a structured evidence object $E_i$. After all documents are processed, the full evidence set $\mathbb{E}$ and the claim's semantic features are passed to the evaluation network, which produces two scores for the current node: a veracity score $v_n$ and a reliability score $r_n$.

\begin{figure}[ht]
   \centering
   \includegraphics[width=0.9\linewidth]{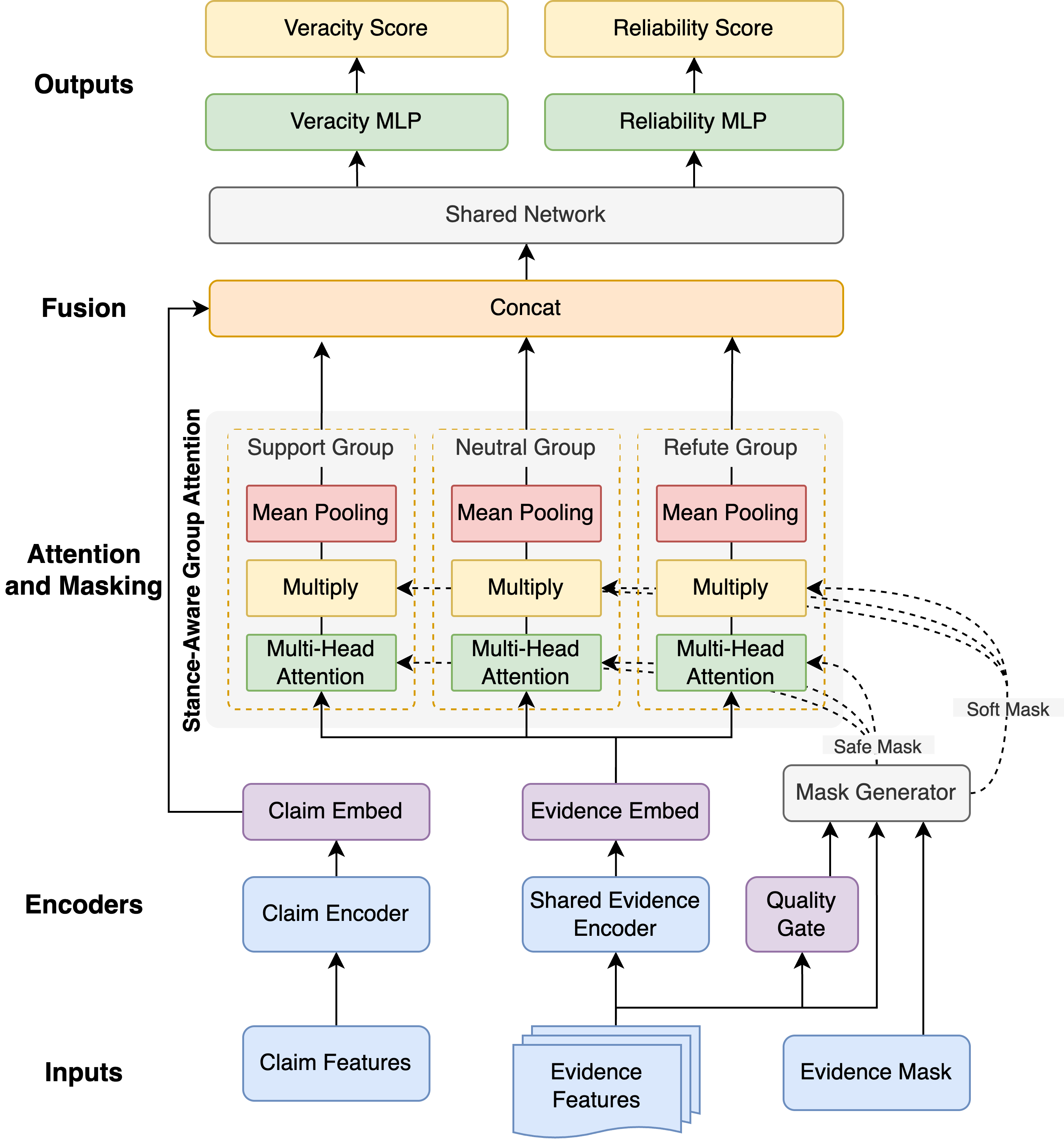}
   \caption{Detail of Evaluation Network.}
   \label{fig:scoring}
\end{figure}

As illustrated in Fig.~\ref{fig:scoring}, the evaluation network jointly predicts veracity and reliability through a stance-aware multi-path attention mechanism. Claim features and evidence features are first encoded by a claim encoder and a shared evidence encoder into high-dimensional claim embeddings and evidence embeddings, respectively. In parallel, the evidence features are also passed through a quality gate, a small multilayer network that produces a per-evidence quality weight reflecting its estimated relevance and credibility.
The encoded evidence then enters the stance-aware group attention module, where evidence objects are partitioned into three parallel processing tracks according to their annotated stance: supporting, neutral, and refuting. A mask generator takes as input the stance labels, the quality gate weights, and a binary evidence mask indicating which evidence slots are occupied, and produces two types of masks for each track. The safe mask is a hard binary mask that blocks invalid or cross-stance evidence from attending to the current track. The soft mask is a continuous weight derived from the quality gate output, used to scale attention outputs by evidence quality. Within each track, the evidence embeddings pass through a multi-head self-attention layer gated by the safe mask, enabling intra-stance contextual interaction. The attention outputs are then multiplied by the soft mask to inject quality priors, and mean pooling compresses the variable-length evidence set within each track into a fixed-dimensional stance aggregation vector.
In the fusion stage, the original claim embedding and the three stance aggregation vectors are concatenated along the feature dimension, forming a compact unified representation that encodes the claim content alongside the evidential signal from all three stances. This representation is passed through a shared network for global feature integration, after which it is routed to two parallel output heads: a veracity MLP and a reliability MLP, producing the final veracity score and reliability score, respectively.

The evaluation network uses a 12-dimensional claim feature vector and a padded $20 \times 8$ evidence matrix with a binary evidence mask. Claim features encode category, locale, and whether the claim describes a recent event, historical fact, or factual statement. Each evidence item encodes stance, source type, relevance, and credibility. Training targets are produced by an LLM judge, which assigns soft veracity and reliability scores based on evidential support, source authority, evidence quantity, cross-source consistency, and directness. Given targets $(\tilde{v}, \tilde{r})$, the network minimizes
\begin{equation}
\mathcal{L} = \text{BCE}(\hat{v},\, \tilde{v}) + \text{BCE}(\hat{r},\, \tilde{r}),
\end{equation}
where $\hat{v}$ and $\hat{r}$ are the predicted veracity and reliability scores. The model is optimized with Adam using a learning rate of $10^{-3}$ and gradient norm clipping at $1.0$.

\subsection{Tree Management}

Once a node obtains its veracity score $v_{n^*}$ and reliability score $r_{n^*}$ from the evaluation network, the tree management module determines how the argument tree should evolve. Algorithm~\ref{alg:toe} summarizes the full procedure, where three binary flags govern the control flow: the expansion flag $\delta_n$, the decisiveness flag $\phi_n$, and the convergence flag $\gamma$.
These flags are defined in terms of five thresholds: the expansion reliability threshold $\tau_r^{\text{exp}}$, the decision reliability threshold $\tau_r^{\text{dec}}$, the upper and lower veracity thresholds $\tau_v^{+}$ and $\tau_v^{-}$, and the convergence tolerance $\tau_c$. Specifically:
\begin{itemize}
    \item $\delta_n = 1$ if $r_n^{\text{self}} < \tau_r^{\text{exp}}$, indicating that the node's self-reliability is too low for a confident judgment and further evidence is needed; the root node is always expanded on its first evaluation regardless of this condition.
    \item $\phi_n = 1$ if $r_n^{\text{aggr}} > \tau_r^{\text{dec}}$ and $v_n^{\text{aggr}} \notin [\tau_v^{-}, \tau_v^{+}]$, indicating that the aggregated evidence is sufficiently reliable and the verdict is unambiguously true or false.
    \item $\gamma = 1$ if the range of $v_0$ over the last three iterations is below $\tau_c$ and $r_0^{\text{aggr}} > \tau_r^{\text{exp}}$, indicating that the root veracity score has stabilized and further iterations are unlikely to change the outcome.
\end{itemize}

\begin{algorithm}[ht]
\caption{Tree of Evidence Verification}
\label{alg:toe}
\textbf{Input}: Initial claim $c_0$, maximum iterations $T_{\max}$, and thresholds $\tau_r^{\text{exp}}, \tau_r^{\text{dec}}, \tau_v^{+}, \tau_v^{-}, \tau_c$.\\
\textbf{Output}: Final root veracity score $v_0$ and reliability score $r_0$.
\begin{algorithmic}[1]
\STATE Initialize tree $\mathcal{T}$ with root node $n_0 \leftarrow c_0$
\STATE Initialize queue $\mathcal{Q} \leftarrow \{n_0\}$
\FOR{$t = 1, 2, \dots, T_{\max}$}
    \IF{$\mathcal{Q}$ is empty}
        \STATE \textbf{break}
    \ENDIF
    \STATE $n^* \leftarrow \arg\max_{n \in \mathcal{Q}} \textsc{Priority}(n)$
    \STATE $\mathcal{Q} \leftarrow \mathcal{Q} \setminus \{n^*\}$
    \STATE $\mathcal{E} \leftarrow \textsc{GetEvidence}(n^*)$
    \STATE $v_{n^*}, r_{n^*} \leftarrow \textsc{Evaluate}(n^*, \mathcal{E})$
    \IF{$\delta_{n^*} = 1$}
        \STATE $\mathcal{S} \leftarrow \textsc{Decompose}(n^*)$
        \STATE $\mathcal{T}.\textsc{Add}(n^*, \mathcal{S})$
        \STATE $\mathcal{Q} \leftarrow \mathcal{Q} \cup \mathcal{S}$
    \ENDIF
    \STATE $v_0, r_0 \leftarrow \textsc{Aggregate}(\mathcal{T})$
    \STATE $\mathcal{T}, \mathcal{Q} \leftarrow \textsc{Prune}(\mathcal{T}, \mathcal{Q})$
    \IF{$\gamma = 1$ \textbf{or} $\phi_0 = 1$}
        \STATE \textbf{break}
    \ENDIF
\ENDFOR
\STATE \textbf{return} $(v_0, r_0)$
\end{algorithmic}
\end{algorithm}

If the expansion flag $\delta_{n^*} = 1$, indicating that the self-reliability $r_{n^*}$ falls below $\tau_r^{\text{exp}}$ and the node is eligible for expansion, the system invokes a decomposition LLM. The LLM breaks the claim into two to four sub-claims, each targeting a distinct verification dimension: geographic accuracy, entity identity, core event, causal relationship, source credibility, quantitative accuracy, or contextual background. Each sub-claim is assigned an importance weight, with all weights summing to one, and the resulting sub-claims are appended to the argument tree $\mathcal{T}$ as children of the current node and added to the processing queue $\mathcal{Q}$.
After any structural change to the tree, scores are propagated bottom-up via the aggregation network. For each node, the network takes as input the node's own veracity and reliability scores alongside those of all processed children that reflect the combined evidence from the entire subtree. This update refreshes the root node's scores $v_0$ and $r_0$ to reflect all available evidence. As illustrated in Fig.~\ref{fig:aggr}, the aggregation network encodes the node's own 16-dimensional feature vector and the feature vectors of its evaluated children, applies self-attention over child representations, and uses cross-attention from the parent node to aggregate child evidence into updated veracity and reliability scores.

\begin{figure}[ht]
   \centering
   \includegraphics[width=1.0\linewidth]{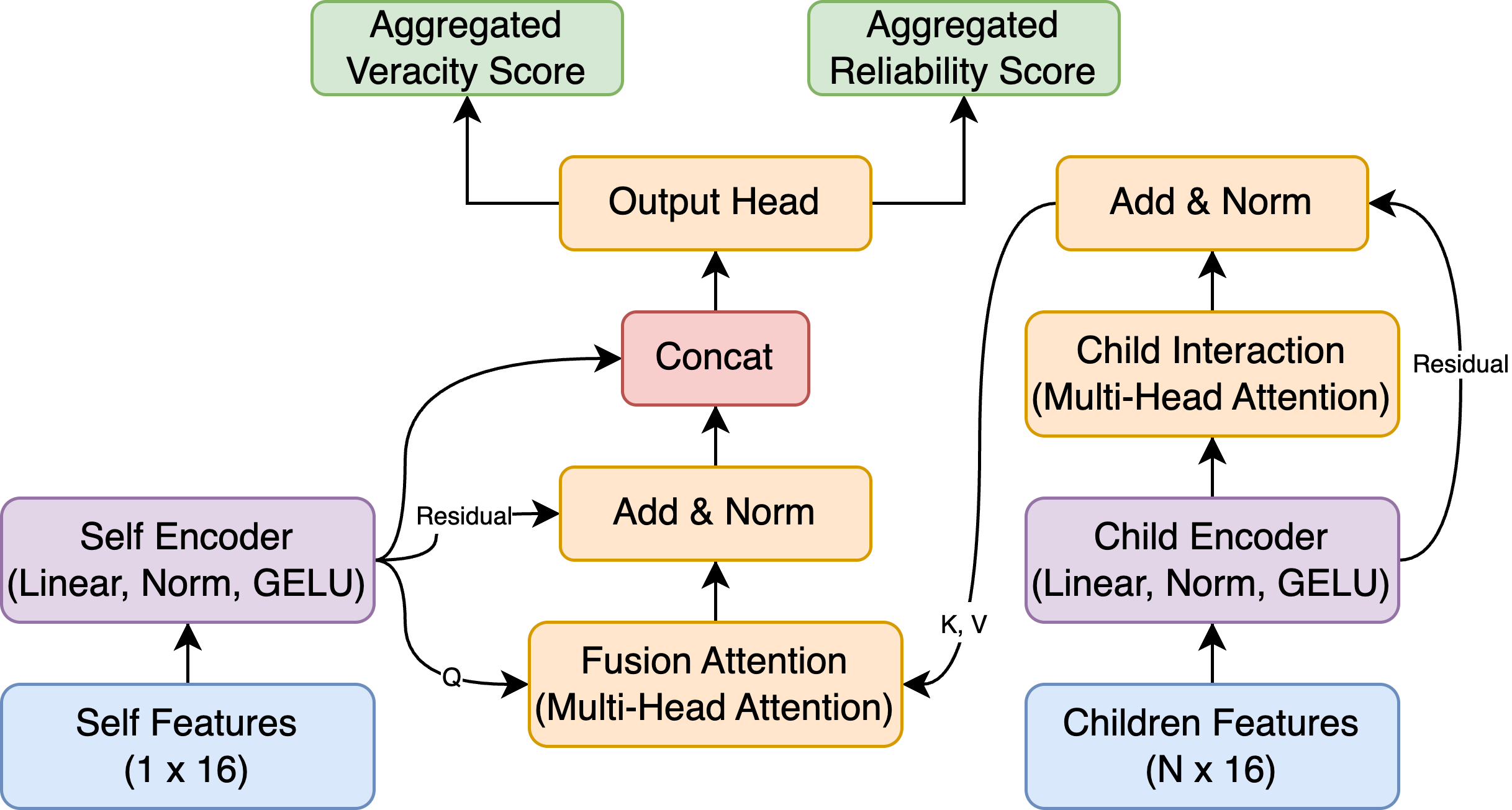}
   \caption{Detail of Tree Aggregation Network.}
   \label{fig:aggr}
\end{figure}

Before neural training, ToE uses a rule-based bottom-up aggregation scheme as an interpretable fallback and a source of soft supervision. Leaf and unexpanded nodes inherit their self-assessed scores. Internal nodes compute veracity as a reliability- and importance-weighted average of child veracity scores, while reliability is adjusted by the number of available children and the consistency of their veracity estimates. The neural aggregator is then trained bottom-up on completed trees. Its loss combines a supervised root loss and an imitation loss over non-root nodes:
\begin{equation}
\mathcal{L}_{\text{root}} = w_{\text{root}} \left[ \text{BCE}(\hat{v}_{\text{root}},\, y_v) + \text{BCE}(\hat{r}_{\text{root}},\, y_r) \right],
\end{equation}
\begin{equation}
\mathcal{L}_{\text{imit}} = w_{\text{imit}} \cdot \frac{1}{|\mathcal{N}|} \sum_{i \in \mathcal{N}} \left[ \text{BCE}(\hat{v}_i,\, \tilde{v}_i) + \text{BCE}(\hat{r}_i,\, \tilde{r}_i) \right],
\end{equation}
with $w_{\text{root}} = 1.0$ and $w_{\text{imit}} = 0.1$. The aggregation network is optimized with AdamW using a learning rate of $10^{-4}$, weight decay $10^{-5}$, and gradient norm clipping at $1.0$.

Following aggregation, the system performs a top-down pruning pass. For any non-leaf node whose decisiveness flag $\phi_n = 1$, meaning its aggregated reliability $r_n^{\text{aggr}}$ exceeds $\tau_r^{\text{dec}}$ and its aggregated veracity $v_n^{\text{aggr}}$ falls outside $[\tau_v^{-}, \tau_v^{+}]$, all pending descendants are removed from $\mathcal{Q}$, avoiding unnecessary computation on branches that no longer affect the outcome.
The loop then checks whether the convergence flag $\gamma = 1$ or the root's decisiveness flag $\phi_0 = 1$. Convergence holds when the range of $v_0$ over the last three iterations drops below $\tau_c$ and $r_0^{\text{aggr}}$ exceeds $\tau_r^{\text{exp}}$. If either condition is met, the system terminates and returns $v_0$ and $r_0$ as the final verdict. Otherwise, the next node is selected from $\mathcal{Q}$ according to its priority score, which is computed as a weighted combination of the parent node's uncertainty and the node's own importance weight assigned at decomposition time, and the next iteration begins.

\section{Evaluation}

\subsection{Experimental Setup}

\textbf{Datasets.}
We evaluate ToE on three publicly available fact-checking benchmarks. The LIAR dataset~\cite{liar} contains 12,836 human-labeled short statements about political news, each annotated with one of six fine-grained veracity labels along with speaker metadata such as party affiliation and job title. PolitiFact~\cite{politifact} comprises 21,152 expert-verified statements sourced from the PolitiFact website, spanning claims made between 2008 and 2022 across six verdict categories. Check-COVID~\cite{check_covid} is a domain-specific benchmark of 1,504 COVID-19 related claims drawn from news sources, each verified against evidence from scientific journal articles and labeled as supported, refuted, or lacking sufficient information.
To enable consistent three-class accuracy evaluation across all datasets, we map each dataset's original labels to a unified label set: true, false, and uncertain. For LIAR and PolitiFact, the labels pants-fire, false, and barely-true are mapped to false; half-true is mapped to uncertain; and mostly-true and true are mapped to true. The two datasets differ only in that PolitiFact uses mostly-false in place of barely-true, which is likewise mapped to false. For Check-COVID, the labels SUPPORT, REFUTE, and NOTENOUGHINFO are mapped to true, false, and uncertain, respectively. To further assess the robustness of ToE against AI-generated misinformation, we construct \textbf{AdvFact}, an adversarial benchmark designed to simulate sophisticated false claims that are stylistically convincing yet factually incorrect. We randomly sample 100 false statements from the LIAR dataset and apply two complementary poisoning strategies to generate 100 samples each: FakeGPT~\cite{fakegpt} rewrites each statement into a more fluent and detail-rich false narrative, while PoisonedRAG~\cite{poisonedrag} injects adversarially crafted passages into the retrieval corpus to mislead evidence-based reasoning. The resulting dataset contains 200 samples in total, all labeled as false, and is evaluated directly without any additional training.

\textbf{Comparison Methods.}
We compare ToE against a range of LLM-based fact-checking methods. Direct prompting feeds the raw claim directly to an LLM and asks for a veracity judgment without any additional context. Z-CoT and DefGen~\cite{f3} are two structured prompting strategies proposed by the same work, eliciting step-by-step reasoning through zero-shot chain-of-thought and deductive generation, respectively, without external retrieval. AFaCTA~\cite{afacta} employs a multi-agent voting mechanism to aggregate judgments from multiple LLM instances. TELLER~\cite{teller} models the verification process through a dual cognitive and decision-making system. STEEL~\cite{steel} performs iterative multi-round retrieval to progressively gather evidence before reaching a verdict. AdSent~\cite{adsent} neutralizes sentiment signals in the input claim to improve veracity judgment. To assess whether the gains of ToE generalize across model scales, we instantiate all methods on two backbone LLMs: DeepSeek-V3.2~\cite{deepseek-v32}, a large-scale model, and gpt-oss-20b~\cite{gpt-oss}, a smaller open-weight reasoning model released by OpenAI.

\textbf{ToE Settings.}
The argument tree is constrained to a maximum depth of 5 and a maximum of $T_{\max} = 20$ iterations. The thresholds are set as $\tau_r^{\text{exp}} = 0.7$ (expansion), $\tau_r^{\text{dec}} = 0.7$ (decision), $\tau_v^{+} = 0.7$, $\tau_v^{-} = 0.3$ (veracity bounds), and $\tau_c = 0.1$ (convergence). For the retrieval agent, each search action returns at most 3 results, and the total number of retrieved documents per node is capped at 10. The final veracity score $v \in [0, 1]$ is mapped to a three-way verdict: $v < 0.4$ is classified as \textit{False}, $0.4 \leq v \leq 0.6$ as \textit{Uncertain}, and $v > 0.6$ as \textit{True}. To validate that our approach is \textbf{dataset-agnostic}, we train exclusively on the training split of the LIAR dataset and evaluate directly on its test split as well as all other datasets without any dataset-specific fine-tuning or adaptation.

\subsection{Results and Analysis}

\begin{table*}[htbp]
    \centering
    \caption{Average Accuracy Comparison of Different Methods Across Datasets and LLMs}
    \label{tab:accuracy_comparison}
    \begin{tabular}{l|cccc|cccc}
        \hline
        \multirow{2}{*}{\textbf{Method}} & \multicolumn{4}{c|}{\textbf{DeepSeek-V3.2}} & \multicolumn{4}{c}{\textbf{GPT-OSS-20B}} \\
        \cline{2-9}
        & \textbf{LIAR} & \textbf{PolitiFact} & \textbf{CheckCOVID}  & \textbf{AdvFact} & \textbf{LIAR} & \textbf{PolitiFact} & \textbf{CheckCOVID} & \textbf{AdvFact} \\
        \hline
        Direct  & 0.38 & 0.32 & 0.28 & 0.16 & 0.24 & 0.40 & 0.26 & 0.08 \\
        AFacTA  & 0.32 & 0.44 & 0.36 & 0.12 & 0.32 & 0.48 & 0.32 & 0.12 \\
        ZCoT    & 0.42 & 0.60 & 0.52 & 0.36 & 0.34 & 0.22 & 0.44 & 0.20 \\
        DefGen  & 0.50 & 0.38 & 0.36 & 0.28 & 0.38 & 0.56 & 0.40 & 0.24 \\
        STEEL   & 0.46 & 0.60 & 0.40 & 0.16 & 0.20 & 0.38 & 0.46 & 0.14 \\
        TELLER  & 0.46 & 0.68 & 0.56 & 0.40 & 0.60 & \textbf{0.64} & \textbf{0.52} & 0.32 \\
        Adsent   & 0.40 & 0.46 & 0.36 & 0.30 & 0.44 & 0.48 & 0.22 & 0.52 \\
        \hline
        \textbf{ToE} & \textbf{0.62} & \textbf{0.72} & \textbf{0.60} & \textbf{0.64} & \textbf{0.68} & 0.62 & 0.48 & \textbf{0.54} \\
        \hline
    \end{tabular}
\end{table*}

Table~\ref{tab:accuracy_comparison} compares ToE against six baseline methods across four datasets and two LLMs. In each trial, 50 claims are sampled from each dataset using a fixed random seed, and all methods are prompted to base their judgments solely on the provided input to ensure a fair comparison. Each experiment is repeated ten times and the median accuracy is reported.
The four datasets cover distinct claim domains, and ToE performs consistently across all of them. The improvement is most visible on AdvFact, indicating that the method generalizes to adversarially framed claims. Notably, the Direct baseline, which prompts the LLM to judge claims without any external retrieval, achieves non-trivial accuracy on several datasets, likely because some claims in these benchmarks overlap with knowledge seen during pretraining, enabling correct judgments from parametric memory alone. The results also suggest that backbone capability influences method effectiveness: ZCoT, which relies heavily on chain-of-thought reasoning, scores 0.60 on PolitiFact under DeepSeek-V3.2 but drops to 0.22 under GPT-OSS-20B, reflecting its sensitivity to the underlying model's reasoning capacity. ToE improves over the Direct baseline across all settings, confirming that structured retrieval and evidence aggregation provide reliable gains beyond parametric knowledge.

\begin{figure}[ht]
   \centering
   \includegraphics[width=1.0\linewidth]{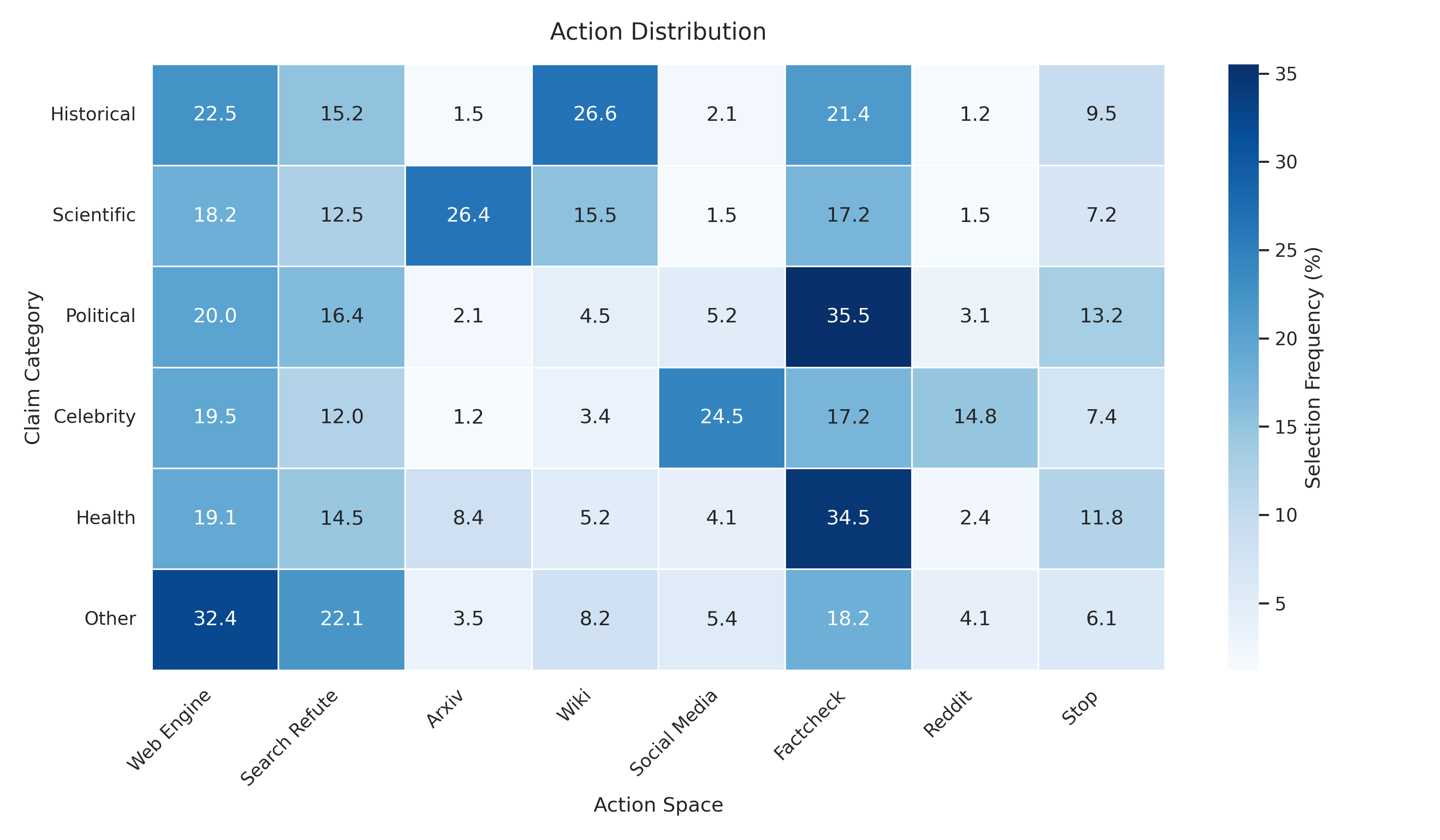}
   \caption{Heatmap of Action Distribution.}
   \label{fig:heatmap}
\end{figure}

\begin{table*}[t]
\centering
\caption{Ablation study of the search tool space.}
\label{tab:action_ablation}
\begin{tabular}{lccccc}
\toprule
\textbf{Metrics} & \textbf{Overall Acc.} & \textbf{Acc. (TRUE)} & \textbf{Acc. (FALSE)} & \textbf{Acc. (UNCERTAIN)}  & \textbf{Avg. Steps} \\
\midrule
Full Tool Space (All Actions) & 80.0\% & 91.6\% & 75.0\% & 66.7\% & 5.4 \\
\midrule
w/o Academic (\textit{ArXiv}) & 73.3\% & 75.0\%  & 83.3\% & 50.0\% & 6.8 \\
w/o Counter-Evidence (\textit{RefuteSearch}) & 63.3\% & 83.3\% & 50.0\% & 50.0\% & 4.0 \\
w/o Factcheck (\textit{Wiki, PolitiFact}) & 56.6\% & 50.0\% & 75.0\% & 33.3\% & 7.8 \\
w/o Social Media (\textit{Twitter, Reddit}) & 63.3\% & 66.6\% & 83.3\% & 16.6\% & 3.8\\
\bottomrule
\end{tabular}
\end{table*}

To investigate the agent's adaptive search behavior, we manually collected 30 claims across six categories. Each category consists of five claims (two true, two false, and one uncertain) to evaluate performance across different veracity labels. run ToE on all samples, and record the frequency with which the retrieval agent selects each evidence source. Figure~\ref{fig:heatmap} reveals clear category-specific patterns. Scientific claims show the highest Arxiv usage at 26.4\%, reflecting the agent's preference for peer-reviewed literature when evaluating technical assertions. Political claims and health claims exhibit the strongest reliance on Factcheck at 35.5\% and 34.5\% respectively, consistent with the availability of expert fact-checking records for these domains. Celebrity claims favor Social Media at 24.5\%, where public figures' statements and reactions are most readily documented. Historical claims distribute usage more evenly across Web Engine, Wiki, and Factcheck, suggesting that verifying historical statements benefits from multiple complementary sources. The Other category shows the highest Web Engine usage at 32.4\%, which is expected given the heterogeneous nature of claims that do not fall into a well-defined domain. Table~\ref{tab:action_ablation} presents an ablation study on the retrieval tool space, where we partition the available tools into four functional groups: academic sources (ArXiv), counter-evidence search (RefuteSearch), fact-checking databases (Wiki and PolitiFact), and social media platforms (Twitter and Reddit). Removing any single group consistently degrades overall accuracy, confirming that each tool category contributes complementary evidence. The most severe drop occurs when fact-checking sources are removed, reducing overall accuracy to 56.6\% and true-claim accuracy to 50.0\%, which suggests that structured fact-checking records provide the most direct and reliable signal for veracity assessment. These patterns indicate that the retrieval agent, trained with rule-based warm-up, learns to allocate search actions in a domain-aware manner rather than applying a uniform retrieval strategy across all claim types and highlight the particular importance of the design of a diverse tool space.

\section{Conclusion}

We presented ToE, a tree-structured argument reasoning framework for automated fact-checking. By organizing evidence into a hierarchical argument tree and guiding retrieval and evidence scoring through a trained agent, ToE decomposes complex claims into verifiable sub-arguments and aggregates evidence. Experiments across multiple datasets and claim domains demonstrate that the approach generalizes consistently, offering a viable path toward evidence-grounded claim verification resilient to GEO poisoning.

\bibliographystyle{IEEEtran}
\bibliography{main}

\end{document}